\def\year{2021}\relax
\documentclass[letterpaper]{article} % DO NOT CHANGE THIS
\usepackage{aaai21}  % DO NOT CHANGE THIS
\usepackage{times}  % DO NOT CHANGE THIS
\usepackage{helvet} % DO NOT CHANGE THIS
\usepackage{courier}  % DO NOT CHANGE THIS
\usepackage[hyphens]{url}  % DO NOT CHANGE THIS
\usepackage{graphicx} % DO NOT CHANGE THIS
\usepackage{natbib}  % DO NOT CHANGE THIS AND DO NOT ADD ANY OPTIONS TO IT
\usepackage{caption} % DO NOT CHANGE THIS AND DO NOT ADD ANY OPTIONS TO IT
\usepackage[switch]{lineno} 

\usepackage{nicefrac}       % compact symbols for 1/2, etc.
\usepackage{booktabs} % for professional tables
\usepackage{amsmath}
\usepackage{amssymb}
\usepackage[ruled,vlined]{algorithm2e}
\usepackage{subcaption}
\usepackage{mathtools}
\usepackage{sidecap}
\usepackage{hyperref}
\usepackage{xcolor}

\newtheorem{theorem}{Theorem}
\newtheorem{lemma}{Lemma}
\newcommand{\BlackBox}{\rule{1.5ex}{1.5ex}}  % end of proof
\newenvironment{proof}{\par\noindent{\bf Proof\ }}{\hfill\BlackBox}%\\[2mm]}

\title{Self-correcting Q-Learning}
\author {
        Rong Zhu\textsuperscript{\rm 1}\thanks{Part of the work was done while at Columbia University\newline},
        Mattia Rigotti \textsuperscript{\rm 2}\\
}
\affiliations {
    \textsuperscript{\rm 1} ISTBI, Fudan University\\
    \textsuperscript{\rm 2} IBM Research AI\\
    rongzhu56@gmail.com, mr2666@columbia.edu
}

\begin{document}
%\linenumbers

\maketitle

\begin{abstract}
The Q-learning algorithm is known to be affected by the \emph{maximization bias}, i.e.\ the systematic overestimation of action values, an important issue that has recently received renewed attention. Double Q-learning has been proposed as an efficient algorithm to mitigate this bias. However, this comes at the price of an \emph{underestimation} of action values, in addition to increased memory requirements and a slower convergence.
In this paper, we introduce a new way to address the maximization bias in the form of a ``self-correcting algorithm'' for approximating the maximum of an expected value. Our method balances the overestimation of the single estimator used in conventional Q-learning and the underestimation of the double estimator used in Double Q-learning. Applying this strategy to Q-learning results in \emph{Self-correcting Q-learning}. We show theoretically that this new algorithm enjoys the same convergence guarantees as Q-learning while being more accurate. Empirically, it performs better than Double Q-learning in domains with rewards of high variance, and it even attains faster convergence than Q-learning in domains with rewards of zero or low variance. These advantages transfer to a Deep Q Network implementation that we call \emph{Self-correcting DQN} and which outperforms regular DQN and Double DQN on several tasks in the Atari 2600 domain.
\end{abstract}

\section{Introduction}
The goal of Reinforcement Learning (RL) is to learn to map situations to actions so as to maximize a cumulative future reward signal \citep{Sutton:2018}. 
Q-learning proposed by \citet{Watkins:89} is one of the most popular algorithms for solving this problem, and does so by estimating the optimal action value function. 
The convergence of Q-learning has been proven theoretically for discounted Markov Decision Processes (MDPs), and for undiscounted MDPs with the condition that all policies lead to a zero-cost absorbing state \cite{WD:92,Tsitsiklis:94,JJS:94}.
Q-learning has been widely successfully deployed on numerous practical RL problems in fields including control, robotics \cite{Kober13} and human-level game play \cite{Minh:15}.

However, Q-learning is known to incur a \emph{maximization bias}, i.e., the overestimation of the maximum expected action value, which can result in poor performance in MDPs with stochastic rewards.
This issue was first pointed out by \citet{TS:93}, and further investigated by \citet{Hasselt:2010} that proposed Double Q-learning as a way of mitigating the problem by using the so-called double estimator, consisting in two separately updated action value functions. 
By decoupling action selection and reward estimation with these two estimators, Double Q-learning addresses the overestimation problem, but at the cost of introducing a systematic underestimation of action values.
In addition, when rewards have zero or low variances, Double Q-learning displays slower convergence than Q-learning due to its alternation between updating two action value functions.

The main contribution of this paper is the introduction of a new \emph{self-correcting estimator} to estimate the maximum expected value.
Our method is based on the observation that in Q-learning successive updates of the Q-function at time steps $n$ and $n-1$ are correlated estimators of the optimal action value
that can be combined into a new ``self-correcting'' estimator, once the resulting combination is maximized over actions.
Crucially, using one value function (at different time steps) allows us to avoid having to update two action value functions.
First, we will show how to combine correlated estimators to obtain a self-correcting estimator that is guaranteed to balance the overestimation of the single estimator and the underestimation of the double estimator.
We then show that, if appropriately tuned,
this strategy can completely remove the maximization bias, which is particularly pernicious in domains with high reward variances. Moreover, it can also attain faster convergence speed than Q-learning in domains where rewards are deterministic or with low variability.
Importantly, Self-correcting Q-learning does not add any additional computational or memory costs compared to Q-learning.
Finally, we propose a Deep Q Network version of Self-correcting Q-learning that we successfully test on the Atari 2600 domain. 

\paragraph{Related work.}

Beside Double Q-learning, other methods have been proposed to reduce the maximization bias, e.g., 
removing the asymptotic bias of the max-operator under a Gaussian assumption \cite{LDP:13}, estimating the maximum expected value by Gaussian approximation \cite{DNR:16}, averaging Q-values estimates \cite{ABS:17}, and softmax Bellman operator \cite{Song:19}, clipping values in an actor-critic setting \cite{Dorka:19}.
Fitted Q-iteration \cite{EGW:05}, Speedy Q-learning \cite{Azar:11}, and Delayed Q-learning \cite{SLWLL:06} are related variations of Q-learning with faster convergence rate.
In its basic form, our approach is a generalization of regular Q-learning. But importantly it can be applied to any other variant of Q-learning to reduce its maximization bias. Finally, DeepMellow \cite{Kim:19} establishes that using a target network in DQN \cite{Minh:15} also reduces maximization bias, and proposes a soft maximization operator as an alternative.

\paragraph{Paper organization.}
In Section \ref{sec:background}, we  review MDPs and Q-learning. 
In Section \ref{sec:bias}, we consider the problem of estimating the maximum expected value of a set of random variables, and review the maximization bias, the single estimator and the double estimator.
Then we propose the self-correcting estimator, and show that this estimator can avoid both the overestimation of the single estimator and the underestimation of the double estimator. 
In Section \ref{sec:SCQ} we apply the self-correcting estimator to Q-learning and propose Self-correcting Q-learning which converges to the optimal solution in the limit.
In Section \ref{sec:deep} we implement a Deep Neural Network version of Self-correcting Q-learning. 
In Section \ref{sec:task} we show the results of several experiments empirically examining these algorithms. 
Section \ref{sec:end} concludes the paper with consideration on future directions.

%%%%%%
\section{Markov Decision Problems and Q-learning}
%\section{Background: Q-learning}
\label{sec:background}

%\subsection{Q-learning}
We will start recalling the application of Q-learning to solve MDPs.
Let $Q_n(s,a)$ denote the estimated value of action $a$ in state $s$ at time step $n$.
An MDP is defined such that the next state $s^{\prime}$ is determined by a fixed state transition distribution $P: S\times A\times S \rightarrow [0,1]$, 
where $P_{ss^{\prime}}(a)$ gives the probability of ending up in state $s^{\prime}$ after performing action $a$ in $s$, and satisfies $\sum\limits_{s^{\prime}}P_{ss^{\prime}}(a)=1$.
The reward $r$ is drawn from a reward distribution with $\text{E}(r|s,a,s^{\prime})=R_{ss^{\prime}}(a)$ for a given reward function $R_{ss^{\prime}}(a)$.
The optimal value function $Q^*(s,a)$ is the solution to the so-called \emph{Bellman equations}:
$Q^*(s,a)=\sum\limits_{s^{\prime}}P_{ss^{\prime}}(a)\left[R_{ss^{\prime}}(a)+\gamma \max_a Q^*(s^{\prime},a)\right] \forall s,a$,
%\label{Bellman}
%\end{align}
where $\gamma\in [0,1)$ is the discount factor.
As a solution of an MDPs, \cite{Watkins:89} proposed Q-learning, which consists in the following recursive update:
\begin{align}
\label{eq:q_update}
Q_{n+1}(s,a)=&Q_n(s,a)\\
&+\alpha(s,a)
\left[r+\gamma \max_a Q_n(s^{\prime},a)-Q_n(s,a)\right].\notag
\end{align}
Notice that in this expression the $\max$ operator is used to estimate the value of the next state.
%i.e., $\max_aQ_n(s_{n+1},a)$ is used to approximate $\max_a\text{E}[Q_n(s_{n+1},a)]$.
%However, $\text{E}[\max_aQ_n(s_{n+1},a)]\geq \max_a\text{E}[Q_n(s_{n+1},a)]$. 
%the use of the max operator to determine the value of the next state can cause large overestimations of the action values. 
Recently, \cite{Hasselt:2010} showed that this use of the maximum value as an approximation for the maximum expected value introduces a \emph{maximization bias}, which results in Q-learning over-estimating action values.

%\section{A Self-correcting estimator in the sample-average method}
\section{Maximum expected value estimation}
\label{sec:bias}

\paragraph{The single estimator.}

We begin by looking at estimating the maximum expected value.  
Consider a set of $M$ random variables $\{Q(a_i)\}_{i=1}^M$.
We are interested in estimating their maximum expected value, $\max_i\text{E}[Q(a_i)]$.
Clearly however, it is infeasible to know $\max_i\text{E}[Q(a_i)]$ exactly in absence of any assumption on their underlying distributions.
One natural way to approximate the maximum expected value is through the maximum $\max_i\{Q(a_i)\}$, which is called the \emph{single estimator}.
As noticed, Q-learning uses this method to approximate the value of the next state by maximizing over the estimated action values.
Although $Q(a_i)$ is an unbiased estimator of $\text{E}[Q(a_i)]$, from Jensen's inequality we get that
$\text{E}[\max_i\{Q(a_i)\}]\geq \max_i\text{E}[Q(a_i)]$,
meaning that the single estimator is positively biased. 
This is the so-called \emph{maximization bias}, which interestingly has also been investigated in fields outside of RL, such as economics \cite{Van:04}, decision making \cite{SW:06}, and auctions \cite{Thaler:88}.

\paragraph{The double estimator.}

The paper \cite{Hasselt:2010} proposed to address the maximization bias by introducing the \emph{double estimator}.
Assume that there are two independent, unbiased sets of estimators of $\{\text{E}[Q(a_i)]\}_{i=1}^M$: $\{Q^A(a_i)\}_{i=1}^M$ and $\{Q^B(a_i)\}_{i=1}^M$.
Let $a_D^*$ be the action that maximizes $\{Q^A(a_i)\}_{i=1}^M$, that is, $Q^A(a_D^*)=\max_i\{Q^A(a_i)\}$. 
The double estimator uses $Q^B(a_D^*)$ as an estimator for $\max_i\text{E}[Q(a_i)]$.  
This estimator is unbiased in the sense that $\text{E}[Q^B(a_D^*)]=\text{E}[Q(a_D^*)]$ due to the independence of $\{Q^A(a_i)\}_{i=1}^M$ and $\{Q^B(a_i)\}_{i=1}^M$. 
However, this estimator has a tendency towards underestimation, i.e., %$$\text{E}[Q^B(a_D^*)]=\sum\limits_{i=1}^M\text{Pr}(i=a_D^*)\text{E}[Q^B(a_i)]\leq \max_i\text{E}[Q^B(a_i)].$$ 
$\text{E}[Q^B(a_D^*)]\leq \max_i\text{E}[Q^B(a_i)]$
(see details in \citealt{Hasselt:2010}).
%\begin{equation}\label{E-double}
%\text{E}Q^B(a^*)=\sum\limits_{i=1}^M\text{Pr}(i=a^*)\text{E}\{Q^B(a)\}
%\end{equation}

\paragraph{The self-correcting estimator.}

Let us now consider two independent and unbiased sets of estimators of $\{\text{E}[Q(a_i)]\}_{i=1}^M$, given by $\{Q^{B_1}(a_i)\}_{i=1}^M$ and $\{Q^{B_{\tau}}(a_i)\}_{i=1}^M$. 
From them, we construct another unbiased set of estimators $\{Q^{B_0}(a_i)\}_{i=1}^M$ by defining
\begin{equation}\label{QeD}
    Q^{B_0}(a_i)=\tau Q^{B_1}(a_i)+(1-\tau) Q^{B_{\tau}}(a_i),
\end{equation}
where $\tau\in[0,1)$  %$\tau\in[0,\infty)$ %
denotes the degree of dependence between $Q^{B_0}(a_i)$ and $Q^{B_1}(a_i)$.
Eqn.~\eqref{QeD} clearly establishes that $Q^{B_0}(a_i)$ and $Q^{B_1}(a_i)$ are non-negatively correlated, unbiased estimators of $\text{E}[Q(a_i)]$. 
Let $\sigma_1^2$ and $\sigma_{\tau}^2$ be the variances of $Q^{B_1}(a_i)$ and $Q^{B_{\tau}}(a_i)$, respectively.
The Pearson correlation coefficient between $Q^{B_0}(a_i)$ and $Q^{B_1}(a_i)$ is 
$\rho=\tau\sigma_1/\sqrt{\tau^2\sigma_1^2+(1-\tau)^2\sigma_{\tau}^2}$.
When $\tau\rightarrow 1$, $Q^{B_0}(a_i)$ is completely correlated with $Q^{B_1}(a_i)$. 
While as $\tau$ becomes smaller, the correlation is weaker. 

Denoting $\beta=1/(1-\tau)$,%$\beta=\frac{1}{1-\tau}$,
we rewrite Eqn.~\eqref{QeD} as
\begin{equation}
\label{Qe}
    Q^{B_{\tau}}(a_i)=Q^{B_1}(a_i)-\beta [Q^{B_1}(a_i)-Q^{B_0}(a_i)].
\end{equation}
Let $a_{\tau}^*$ indicate the action maximizing $\{Q^{B_{\tau}}(a_i)\}_{i=1}^M$, i.e.\ 
$a_{\tau}^*=\arg\max_{a_i} Q^{B_{\tau}}(a_i)$.
We call $Q^{B_0}(a_{\tau}^*)$ \emph{self-correcting estimator} of $\text{E}[Q(a_i)]$ because in a Q-learning setting the roles of $Q^{B_0}(a_i)$ and $Q^{B_1}(a_i)$ are going to be taken up by sequential terms of Q-learning updates (see next section).

\begin{lemma}\label{lemma-compare}
Consider a set of $M$ random variables $\{Q(a_i)\}_{i=1}^M$ with the expected values $\{\text{E}[Q(a_i)]\}_{i=1}^M$. 
Let $\{Q^{B_0}(a_i)\}_{i=1}^M$, $\{Q^{B_1}(a_i)\}_{i=1}^M$, and $\{Q^{B_{\tau}}(a_i)\}_{i=1}^M$ be unbiased sets of estimators satisfying the relation Eqn.~\eqref{QeD}, and $a_{\tau}^*$ the action that maximizes $\{Q^{B_{\tau}}(a_i)\}_{i=1}^M$. 
Assume that $\{Q^{B_{\tau}}(a_i)\}_{i=1}^M$ are independent from $\{Q^{B_1}(a_i)\}_{i=1}^M$.
Then 
\begin{equation*}%\label{Between}
    E[Q^{B_1}(a_{\tau}^*)]\leq E[Q^{B_0}(a_{\tau}^*)]\leq \text{E}[\max_i Q^{B_{\tau}}(a_i)].
\end{equation*}
Furthermore, there exists a $\beta$ such that $E[Q^{B_0}(a_{\tau}^*)]=\max_i\text{E}[Q(a_i)]$.
\end{lemma}

\begin{proof}
The proof is provided in the Appendix.
\end{proof}

Notice that, under the assumption that $\{Q^{B_{\tau}}(a_i)\}_i$ are independent from $\{Q^{B_1}(a_i)\}_i$, by construction $Q^{B_1}(a_{\tau}^*)$ is a \emph{double estimator} of $\max_i\text{E}[Q(a_i)]$.
Lemma \ref{lemma-compare} then establishes that the bias of $Q^{B_0}(a_{\tau}^*)$ is always  between the positive bias of the single estimator $\max_i Q^{B_{\tau}}(a_i)$ and the negative bias of the double estimator $Q^{B_1}(a_{\tau}^*)$. In other words, $Q^{B_0}(a_{\tau}^*)$ is guaranteed to balance the overestimation of the single estimator and the underestimation of the double estimator. 
Therefore, the self-correcting estimator can reduce the maximization bias, and even completely remove it if the parameter $\beta$ is set appropriately.

Let us denote with $\beta^*$ such an ideal (and in general unknown) parameter for which the self-correcting estimator is unbiased. 
A value $\beta^*\rightarrow 1$ indicates that no bias needs to be removed from $Q^{B_0}(a_{\tau}^*)$.
While as $\beta^*$ becomes larger, progressively more bias has to be removed.
Thus, $\beta^*$ can be seen as a measure of the severity of the maximization bias, weighting how much bias should be removed. 
As remarked, it is in practice impossible to know the ideal $\beta^*$.
But these observations tell us that larger biases will have to be corrected by choosing correspondingly larger values of $\beta$.
%Luckily our self-correcting Q-learning in Section \ref{sec:SCQ} is not sensitive to $\beta$. 
%Consequently, larger $\beta^*$ will have to be chosen to correct a larger bias.

\iffalse  %%%%%%%%%%%%
Note that Eqn. (\ref{Qe2}) can be considered as an assumption of correlation structure of the vector $(Q^{B}(a_i), Q^{B_1}(a_i))^{\top}$, and 
$\tau$ denotes the dependence of  $Q^{B}(a_i)$ and $Q^{B_1}(a_i)$.
\fi  %%%%%%%%%%%%

%%%%%
\section{Self-correcting Q-learning}
\label{sec:SCQ}

In this section we apply the self-correcting estimator to Q-learning, and propose a novel method to address its maximization bias. 
We consider sequential updates of the action value function, $Q_{n}(s^{\prime},a)$ and $Q_{n+1}(s^{\prime},a)$, as candidates for the correlated estimators $Q^{B_0}(a_i)$ and $Q^{B_1}(a_i)$ in Lemma \ref{lemma-compare},  
despite the relationship between $Q_{n}(s^{\prime},a)$ and $Q_{n+1}(s^{\prime},a)$ in Q-learning being seemingly more complicated than that defined in Eq.~\eqref{QeD}. 
Replacing $Q^{B_0}(a_i)$ and $Q^{B_1}(a_i)$ in Eq.~\eqref{Qe} with $Q_{n}(s^{\prime},a)$ and $Q_{n+1}(s^{\prime},a)$ gives
\begin{equation*}
    Q_{n+1}^{\beta}(s^{\prime},a)=Q_{n+1}(s^{\prime}, a)-\beta[Q_{n+1}(s^{\prime}, a)-Q_{n}(s^{\prime}, a)].
\label{eq:scq0}
\end{equation*}
Let $a_{\tau}^*$ be the action that maximizes $Q_{n+1}^{\beta}(s^{\prime},a)$ over $a$. 
Following Lemma \ref{lemma-compare}, $Q_{n}(s^{\prime},a_{\tau}^*)$ balances the overestimation of the single estimator and the underestimation of the double estimator, and moreover there exists an optimal value of $\beta$ for which it is unbiased.
However, $Q_{n+1}(s^{\prime},a)$ is not available at time step $n$. 
To address this issue, we construct an alternative Q-function by replacing the sequential updates $Q_{n}(s^{\prime},a)$ and $Q_{n+1}(s^{\prime},a)$ with the sequential updates $Q_{n-1}(s^{\prime},a)$ and $Q_{n}(s^{\prime},a)$ at the previous update step. Specifically, we define the following Q-function:
\begin{equation}
    Q_n^{\beta}(s^{\prime},a)=Q_n(s^{\prime}, a)-\beta[Q_{n}(s^{\prime}, a)-Q_{n-1}(s^{\prime}, a)],
\label{eq:scq}
\end{equation}
where $\beta\geq 1$ is a constant parameter tuning the bias correction.
Therefore, we propose to use $Q_n^{\beta}(s^{\prime},a)$ for action selection according to 
$\hat{a}_{\beta}=\arg\max_aQ_n^{\beta}(s^{\prime},a)$, and to use $Q_n(s^{\prime},\hat{a}_{\beta})$ to estimate the value of the next step.
This results in the following \emph{self-correcting estimator} approximating the maximum expected value:
$Q_n(s^{\prime},\hat{a}_{\beta})\approx \max_a\text{E}[Q_n(s^{\prime},a)]$.
Thus, we propose to replace Eqn.\ \eqref{eq:q_update} from Q-learning with the following novel updating scheme:
\begin{align}\label{update}
Q_{n+1}(s,a)=&Q_n(s,a)\\
&+\alpha_n(s,a)\left[r+\gamma Q_n\left(s^{\prime},\hat{a}_{\beta}\right)-Q_n(s,a)\right].\notag
\end{align}
We call this \emph{Self-correcting Q-learning}, and summarize the method in Algorithm \ref{alg:m-Q}.

\begin{algorithm}
\textbf{Parameters:} step size $\alpha\in (0,1]$, discount factor $\gamma \in (0,1]$, small $\epsilon>0$, and $\beta\geq 1$. \\ %such as taking 1, 2, or 4.\\
\textbf{Initialize} $Q_0(s,a)=0$ for all $a\in A$, and $s$ terminal\\
\textbf{Loop for each episode:}\\
(1) Initialize $s$\\
(2) Loop for each step of episode:\\
(2.a) Choose $a$ from $s$ using the policy $\epsilon$-greedy in $Q$\\
(2.b) Take action $a$, observe $r$, $s^{\prime}$, update $Q_{n}(s,a)$:
\begin{align*}%\label{update}
&\ \ Q_n^{\beta}(s^{\prime},a)=Q_n(s^{\prime}, a)-\beta[Q_n(s^{\prime}, a)-Q_{n-1}(s^{\prime}, a)]\notag\\
&\ \ \hat{\alpha}_{\beta}=\arg\max_a Q_n^{\beta}(s^{\prime},a)\notag\\
&\ \ Q_{n+1}(s,a)=Q_n(s,a)\notag\\
&\ \ \ \ \ \ \ \ \ \ \ \ \ \ \ \ \ \ \ \ \ \ \ +\alpha_n(s,a)\left[r+\gamma Q_n\left(s^{\prime},\hat{\alpha}_{\beta}\right)-Q_n(s,a)\right]\notag\\
%&Q_{n+1}(s,a)=Q_n(s,a)+\alpha_n(s,a)\left[r+\gamma Q_n\left(s^{\prime}, \arg\max_a Q_n^{\tau}(s^{\prime},a)\right)-Q_n(s,a)\right]\\
&\ \ s\leftarrow s^{\prime}\notag
\end{align*}
(3) until $s$ is terminal.
\caption{Self-correcting Q-learning. 
%Note: $\max_a Q_n\left(s^{\prime}, a\right)$ is used in Q-learning. 
}
\label{alg:m-Q}
\end{algorithm}

%%% remark
\paragraph{Remarks.}
$Q_n^{\beta}(s^{\prime},a)=Q_{n-1}(s^{\prime}, a)$ if $\beta=1$.
In this case, the algorithm uses $Q_{n-1}$ from the previous time step instead of $Q_{n}$ to select the action.
This is different from Double Q-learning which trains two Q-functions, but is reminiscent of using a ``delayed'' target network \citep{Minh:15}.
% Although no bias is being corrected when $\beta=1$ according to Lemma \ref{lemma-compare}, empirically we noted that Self-correcting Q-learning still performs better than Q-learning.   
%\end{remark}

%\subsection{Convergence}
We now prove that asymptotically Self-correcting Q-learning converges to the optimal policy.
Comparing Eqns.\ \eqref{eq:q_update} and \eqref{update}, we see that the difference between Self-correcting Q-learning and regular Q-learning is due to $Q_n^{\beta}(s^{\prime},a)$ being different from $Q_n(s^{\prime},a)$.
As the gap between $Q_n(s^{\prime}, a)$ and $Q_{n-1}(s^{\prime}, a)$ becomes smaller, less bias is self-corrected.
This suggests that the convergence of Self-correcting Q-learning for $n\rightarrow \infty$ can be proven with similar techniques as Q-learning.

We formalize this intuition in a theoretical result that follows the proof ideas used by \cite{Tsitsiklis:94} and \cite{JJS:94} to prove the convergence of Q-learning, which are in turn built on the convergence property of stochastic dynamic approximations.
Specifically, Theorem \ref{them-SCQ} below claims that the convergence of Self-correcting Q-learning holds under the same conditions as the convergence of Q-learning. The proof is in the Appendix.

\begin{theorem}
\label{them-SCQ}
If the following conditions are satisfied:
(C1) The MDP is finite, that is, the state and action spaces are finite;
(C2) $\alpha_n(s,a)\in (0,1]$ is such that  $\sum_{n=1}^{\infty}\alpha_n(s,a)=\infty$ and $\sum_{n=1}^{\infty}[\alpha_n(s,a)]^2<\infty$, $\forall s,a$;
(C3) $\text{Var}(r)<\infty$;
%(C4) $\gamma=1$, and all policies lead to a cost free terminal state in probability 1;
(C4) $1\leq\beta<\infty$;
then $Q_n(s,a)$ as updated by Self-correcting Q-learning (Algorithm \ref{alg:m-Q}),
will converge to the optimal value $Q^*$ defined by the Bellman optimality given by the \emph{Bellman equations} with probability one.
%eq.\ \eqref{
\end{theorem}

We conclude this section by discussing the parameter $\beta$ in Self-correcting Q-learning. 
In estimating the maximum expected value, Lemma \ref{lemma-compare} shows that $\beta$ relies on the correlation between $Q^{B_0}(a_i)$ and $ Q^{B_1}(a_i)$.
However, the relation between $Q_n(s^{\prime},a)$ and $Q_{n-1}(s^{\prime},a)$ in Q-learning is more complicated than the setting of the Lemma.
Therefore, the significance of Lemma \ref{lemma-compare} lies in the fact that
$Q_n(s^{\prime},a)-Q_{n-1}(s^{\prime},a)$ can be used as ``direction'' for removing the maximization bias in the objective to search policy,
and that 
$\beta$ can be tuned to approximate the premise of the Lemma and match the correlation between the estimators $Q_n(s^{\prime},a)$ and $Q_{n-1}(s^{\prime},a)$, corresponding to $Q^{B_0}(a_i)$ and $Q^{B_1}(a_i)$ of the self-correcting estimator.

As we will show empirically in Section \ref{sec:task}, the correction of the maximization bias is robust to changes in values of $\beta$.
As rule of thumb, we recommend setting $\beta\approx2, 3, \text{or } 4$, keeping in mind that as reward variability increases (which exacerbates the maximization bias), improved performance may be obtained by setting $\beta$ to larger values.

\section{Self-correcting Deep Q-learning}
\label{sec:deep}

Conveniently, Self-correcting Q-learning is amenable to a Deep Neural Network implementation, similarly to how Double Q-learning can be turned into Double DQN \citep{HGS:16}.
This will allow us to apply the idea of the self-correcting estimator in high-dimensional domains, like those that have been recently solved by Deep Q Networks (DQN), the combination of Q-learning with Deep Learning techniques \citep{Minh:15}.
A testbed that has become standard for DQN is the Atari 2600 domain popularized by the ALE Environment \citep{Bellemare2013}, that we'll examine in the Experiments section.

We first quickly review Deep Q Networks (DQN).
A DQN is a multi-layer neural network parameterizing an action value function  $Q(s,a; \theta)$, where $\theta$ are the parameters of the network that are tuned by gradient descent on the Bellman error.
An important ingredient for this learning procedure to be stable is the use
proposed in \cite{Minh:15} of a target network, i.e.\ a network with parameters $\theta^{-}$ (using the notation of \cite{HGS:16}) which are a delayed version of the parameters of the online network.
Our DQN implementation of Self-correcting Q-learning also makes use of such a target network.

Specifically, our proposed \emph{Self-correcting Deep Q Network algorithm} (ScDQN) equates the current and previous estimates of the action value function $Q_n(s^{\prime}, a)$ and $Q_{n-1}(s^{\prime}, a)$ in Eqn.\ \eqref{eq:scq}
to the target network $Q(s,a; \theta^-)$ and the online network $Q(s,a; \theta)$, respectively. In other words, we compute
\begin{equation*}
Q^{\beta}(s^{\prime},a)=Q(s^{\prime}, a; \theta^-)-\beta[Q(s^{\prime}, a; \theta^-)-Q(s^{\prime}, a; \theta)],
%\label{eq:deep_scq}
\end{equation*}
which, analogously to Algorithm \ref{alg:m-Q}, is used for action selection:
$\hat{a}_{\beta}=\arg\max_a Q^{\beta}(s^{\prime},a)$,
while the target network $Q(s^{\prime}, a; \theta^-)$ is used for action evaluation as in regular DQN.
Everything else also proceeds as in DQN and Double-DQN.

\paragraph{Remarks.}
It's worth drawing a parallel between the relation of ScDQN with Self-correcting Q-learning, and that of regular DQN with Double-DQN.
First, unlike Self-correcting Q-learning which uses $Q_n(s^{\prime}, a)-Q_{n-1}(s^{\prime}, a)$ to correct the maximization bias, ScDQN uses $Q(s^{\prime}, a; \theta^-)-Q(s^{\prime}, a; \theta)$, and therefore takes advantage of the target network introduced in DQN.
This is analogous to Double-DQN performing action selection through the target network, instead of using a second independent Q function like vanilla Double Q-learning.
If memory requirements weren't an issue, a closer implementation to Self-correcting Q-learning would be to define
$Q^{\beta}(s^{\prime},a)=Q(s^{\prime}, a; \theta^-)-\beta[Q(s^{\prime}, a; \theta^-)-Q(s^{\prime}, a; \theta^{=})]$, where  $\theta^{=}$ denotes a set of parameters delayed by a fixed number of steps.
This would be an alternative strategy worth investigating.
Second, ScDQN is implemented such that with $\beta=0$ it equals regular DQN, and with $\beta=1$ it goes to Double-DQN.
The intuition that this provides is that ScDQN with $\beta\geq 1$ removes a bias that is estimated in the direction between DQN and Double-DQN, rather then interpolating between the two.

In the Experiments section we benchmark Self-correcting Q-learning on several classical RL tasks, and ScDQN on a representative set of tasks of the Atari 2600 domain.

%%%%%%
\begin{figure}[ht]
    \centering
    \includegraphics[width=0.45\textwidth]{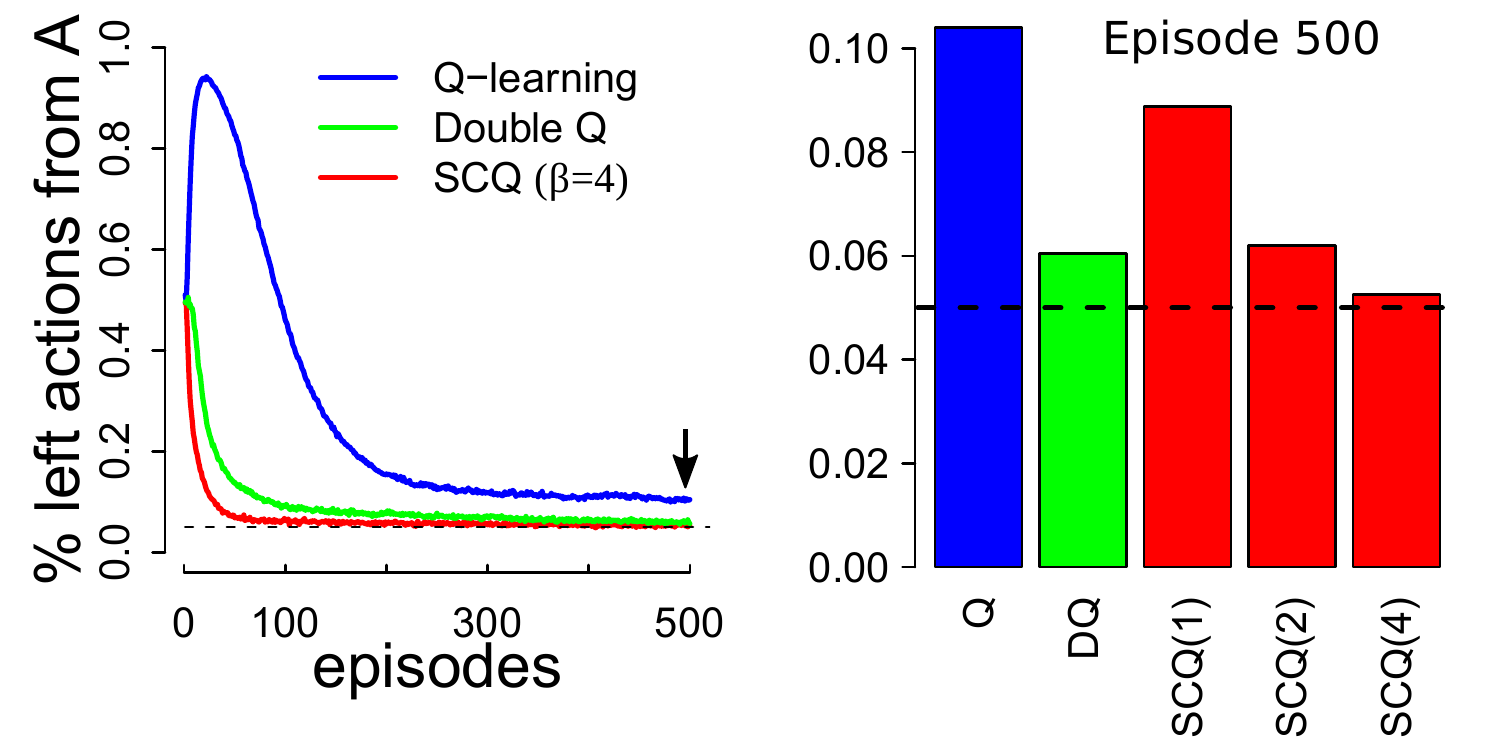}
    \caption{
       Maximization bias example. 
        \textbf{Left:} percent of left actions taken as a function of episode.   
        Parameter settings are $\epsilon=0.1$, $\alpha=0.1$, and $\gamma=1$. Initial action value estimates are zero. \textbf{Right:} percent of left actions from A, averaged over the last five episodes (arrows) for Q-learning, Double Q-learning and Self-correcting Q-learning with increasing $\beta$, which decreases bias.
        Results averaged over 10,000 runs. 
    } \label{fig:max}
\end{figure}
%%%%%%

%%%%%%
\section{Experiments}
\label{sec:task}

%%%%%
\begin{figure*}[htb]
    \centering
    \includegraphics[keepaspectratio,width=0.8\linewidth]{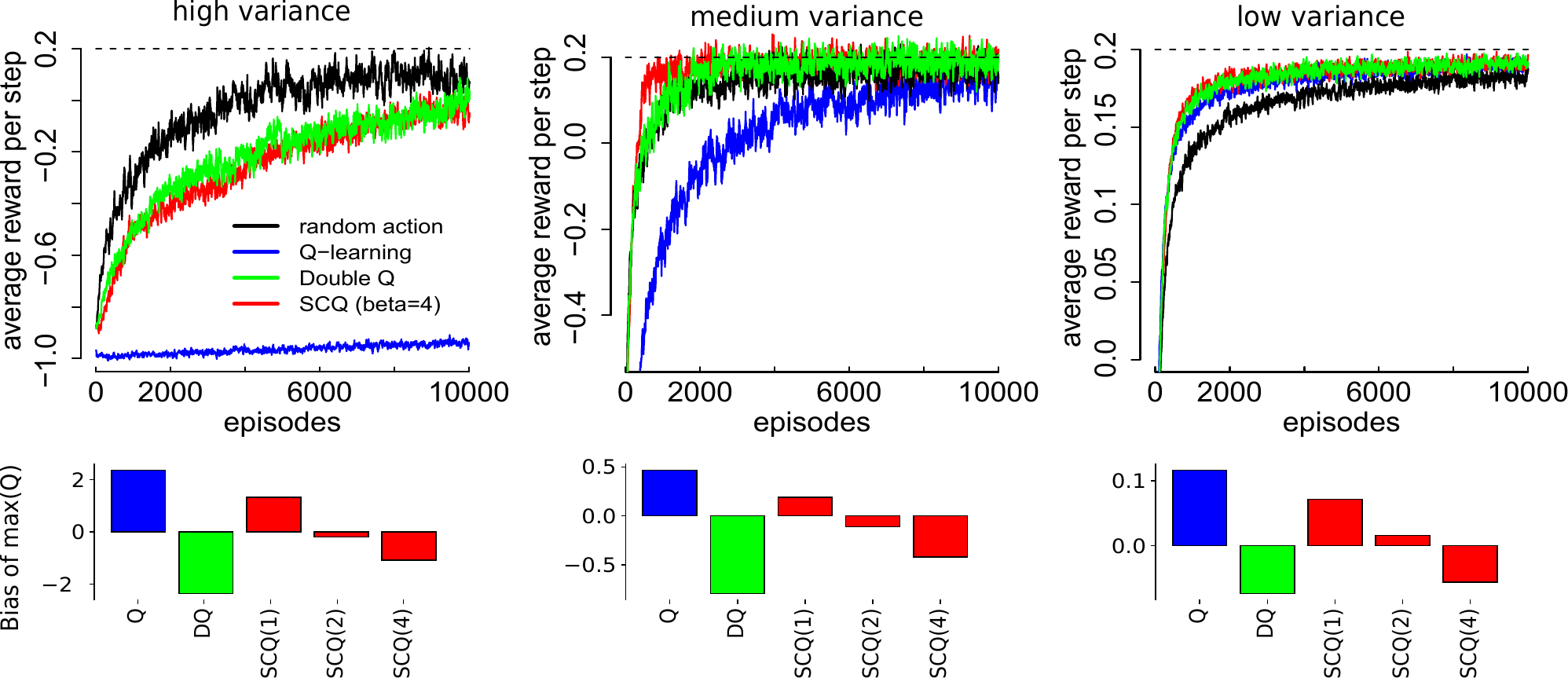}
    \caption{Grid-world task. Rewards of a non-terminating step are uniformly sampled between $(-12, 10)$ (high variance), $(-6, 4)$ (medium variance), and $(-2, 0)$ (low variance). 
    \textbf{Upper row:} average reward per time step. \textbf{Lower row:} bias of the maximal action values of the final episode in the starting state, where  $\text{SCQ}(c)$, c=1,2,4, denotes $\text{SCQ}(\beta=c)$. %Random action policy (not shown in the bar plots) has bias values of -4.11, -2.24, and -1.24 for high, medium and low variance, respectively. 
    Average rewards are accumulated over 500 rounds and averaged over 500 runs. 
    }
    \label{fig:grid}
\end{figure*}
%%%%%

We compare in simulations the performance of several algorithms: \emph{Q-learning}, \emph{Double Q-learning}, and our \emph{Self-correcting Q-learning} (denoted as \emph{SCQ} in the figures), with $\beta=1,2,4$.
Note that Self-correcting Q-learning can be applied to debias any variant of Q-learning, but for simplicity in the empirical studies we only apply our method to Q-learning and focus on the comparison of the resulting self-correcting algorithm with Q-learning and Double Q-learning.
%As a baseline we also consider a random policy. 
We consider three simple but representative tasks.
First, the maximization bias example shows that Self-correcting Q-learning can remove more bias than Double Q-learning. 
Second, the grid-world task serves to establish the advantage of Self-correcting Q-learning over Double Q-learning in terms of overall performance, and shows its robustness towards high reward variability.
Lastly, the cliff-walking task shows that Self-correcting Q-learning displays faster convergence than Q-learning when rewards are fixed.

\paragraph{Maximization Bias Example.}

This simple episodic MDP task has two non-terminal states A and B (see details in Figure 6.5 of \cite{Sutton:2018}).
The agent always starts in A with a choice between two actions: left and right. 
The right action transitions immediately to the terminal state with a reward of zero. 
The left action transitions to B with a reward of zero. 
From B, the agent has several possible actions, all of which immediately transition to the termination with a reward drawn from a Gaussian $\mathcal{N}(-0.1, 1)$. 
As a result, the expected reward for any trajectory starting with left is $-0.1$ and that for going right is 0. 
In this settings, the optimal policy is to choose action left 5\% from A.

Fig.\ \ref{fig:max} shows that Q-learning initially learns to take the left action much more often than the right action, and asymptotes to taking it about 5\% more often than optimal, a symptom of the maximization bias.
Double Q-learning has a smaller bias, but still takes the left action about 1\% more often than optimal.
%Random action almost avoids the bias as expected. 
Self-correcting Q-learning performs between Double Q-learning and Q-learning when $\beta=1$, performs similarly to Double Q-learning when $\beta=2$, and almost completely removes the bias when $\beta=4$, demonstrating almost complete mitigation of the maximization bias.

\begin{figure*}[htb]
\centering
\includegraphics[keepaspectratio,width=0.78\linewidth]{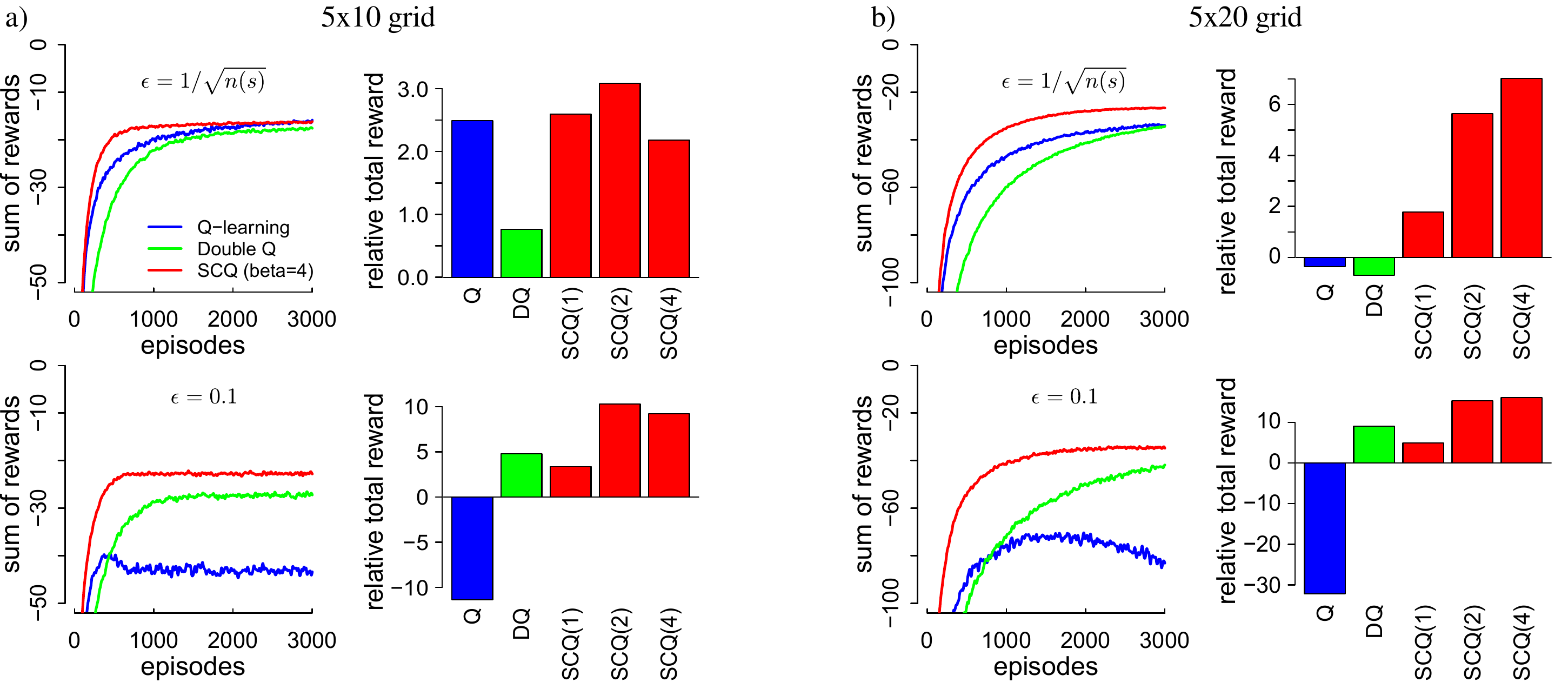}
\caption{Cliff-walking task. 
Two environments are considered: %$5\times5$ space in panel (a), 
a $5\times10$ arena (a), and a $5\times20$ arena (b).
\textbf{Left} in each panel: cumulative rewards. \textbf{Right} in each panel: relative total reward (i.e.\ the difference in total reward from random action) at the final episode.
$\text{SCQ}(c)$, c=1,2,4, denotes $\text{SCQ}(\beta=c)$. 
Each panel shows two $\epsilon$-greedy exploration strategies: $\epsilon=1/\sqrt{n(s)}$ (upper panel), and $\epsilon=0.1$ (lower panel), where $n(s)$ is the number of times state $s$ has been visited.  
The step sizes are chosen: $\alpha_n(s,a)=0.1(100+1)/(100+n(s,a))$, where $n(s,a)$ is the number of updates of each state-action.
Data points are averaged over 500 runs, then are smoothed for clarity.
}
\label{fig:walk}
\end{figure*}
%%%%

\paragraph{Grid-world task.} %: effect of the randomness of reward}

We follow the $3\times3$ grid-world MDP in \cite{Hasselt:2010}, but study different degrees of reward randomness. 
The starting state is in the southwest position of the grid-world and the goal state is in the northeast. 
Each time the agent selects an action that puts it off the grid, the agent stays in the same state.
In this task each state has 4 actions, i.e.\ the 4 directions the agent can go.
In each non-terminating step, the agent receives a random reward uniformly sampled from an interval $(L, U)$, which we choose to modulate the degree of randomness.
We consider three intervals:
$(-12, 10)$ (high variability), $(-6, 4)$ (medium variability), and $(-2, 0)$ (low variability).
Note that all 3 settings have the same optimal values.
In the goal state any action yields $+5$ and the episode ends. 
The optimal policy ends one episode after five actions, so that the optimal average reward per step is $+0.2$. 
Exploration is encouraged with $\epsilon$-greedy action selection with $\epsilon(s)= 1/\sqrt{n(s)}$, where $n(s)$ is the number of times state $s$ has been visited. 
The learning rate is set to $\alpha(s, a) = 1/n(s, a)$, where $n(s,a)$ is the number of updates of each state-action. 
For Double Q-learning, both value functions are updated for each state-action.

The upper panels of Fig.\ \ref{fig:grid} show the average reward per time step, while the lower panels show the deviation of the maximal action value from optimal (the bias) after 10,000 episodes.
First, Self-correcting Q-learning with $\beta=2$ gets very close to the optimal value of the best action in the starting state (the value is about 0.36). 
Self-correcting Q-learning with $\beta=1$ still displays overestimation which, however, is much smaller than Q-learning.
Self-correcting Q-learning with $\beta=4$ shows a small underestimation which, however, is much smaller than that of Double Q-learning.
These observations are consistent for different degrees of reward randomness.
This supports the idea that Self-correcting Q-learning can balance the overestimation of Q-learning and the underestimation of Double Q-learning.
Second, Self-correcting Q-learning performs as well as Double Q-learning in terms of average rewards per step. 
Comparing performance under high, medium, and low reward variability, we observe the following. 
%the differences in these methods become small when the variability decreases. 
When variability is high, Self-correcting Q-learning performs as well as Double Q-learning, 
%and interestingly, the random policy shows the best performance, 
and Q-learning the worst.
When variability is moderate, Self-correcting Q-learning can performs a little better than Double Q-learning, and Q-learning still performs the worst.
When variability is low, that is, the effect of the maximization bias is low, the difference between all methods becomes small. % and the random action is the worst.
Third, %although $\beta$ is a new parameter to introduce, 
the performance of Self-correcting Q-learning is robust to changes in $\beta$. 
For moderate reward variance, $\beta=2-4$ is a reasonable choice.
As reward variability increases, larger $\beta$ may be better.

%%%%%
\paragraph{Cliff-walking Task.}

Fig.\ \ref{fig:walk} shows the results on the cliff-walking task Example 6.6 in \cite{Sutton:2018}, a standard undiscounted episodic task with start and goal states, and four movement actions: up, down, right, and left.  %The grid world is shown in the figure.   
Reward is $-1$ on all transitions except those into the ``Cliff'' region (bottom row except for the start and goal states).  
If the agent steps into this region, she gets a reward of $-100$ and is instantly sent back to the start.
We vary the environment size by considering a $5\times 10$ and a $5\times 20$ grid.

We measure performance as cumulative reward during episodes, and report average values for 500 runs in Fig.\ \ref{fig:walk}. 
We investigate two $\epsilon$-greedy exploration strategies: $\epsilon=1/\sqrt{n(s)}$ (annealed) 
and $\epsilon=0.1$ (fixed). 
With rewards of zero variance, Double Q-learning shows no advantage over Q-learning,
while Self-correcting Q-learning learns the values of the optimal policy and performs better, with an even larger advantage as the state space increases.
This is consistent for both exploration strategies. 
Conversely, Double Q-learning is much worse than Q-learning when exploration is annealed. 
This experiments indicate that Double Q-learning may work badly when rewards have zero or low variances.
This might be due to the fact that Double Q-learning successively updates two Q-functions in a stochastic way.
We also compare the performance of Self-correcting Q-learning under various $\beta$ values. 
Self-correcting Q-learning performs stably over different $\beta$ values, and works well for $\beta$ between 2 and 4. 
Finally,
larger $\beta$ results in better performance for environments with larger number of states.

\paragraph{DQN Experiments in the Atari 2600 domain.}

\begin{figure*}[htb]
    \centering
	\includegraphics[width=0.8\linewidth]{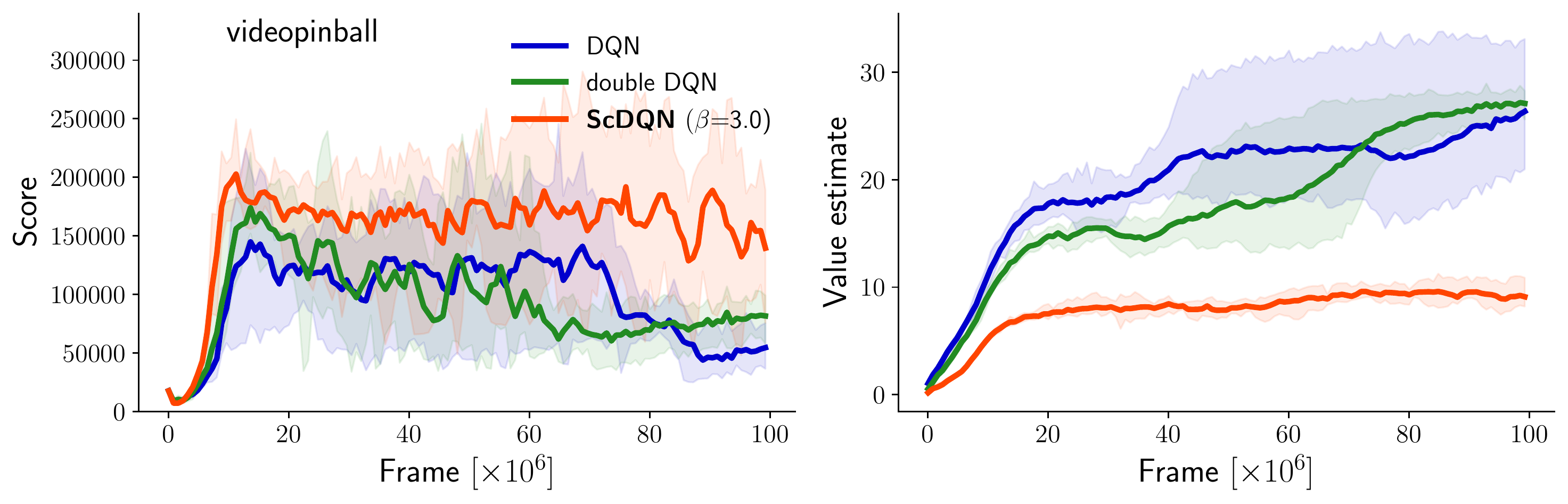}
	\caption{Atari 2600 videopinball game. Results are obtained by running DQN, Double DQN, and ScDQN with 6 different random seeds and hyper-parameters from \cite{Minh:15}. Lines show average over runs and the shaded areas indicate the min and max values of average performance in intervals of 80000 frames over the 6 runs. \textbf{Left plot}: ScDQN quickly reaches a higher score than DQN and Double DQN. DQN and Double DQN catch up, but their scores are unstable and drop when they excessively overestimate values measured in terms of the estimated value of the selected action (\textbf{right plot}).}
	\label{fig:videopinball}
\end{figure*}

\begin{figure*}[htb]
\centering
    \includegraphics[width=0.8\linewidth]{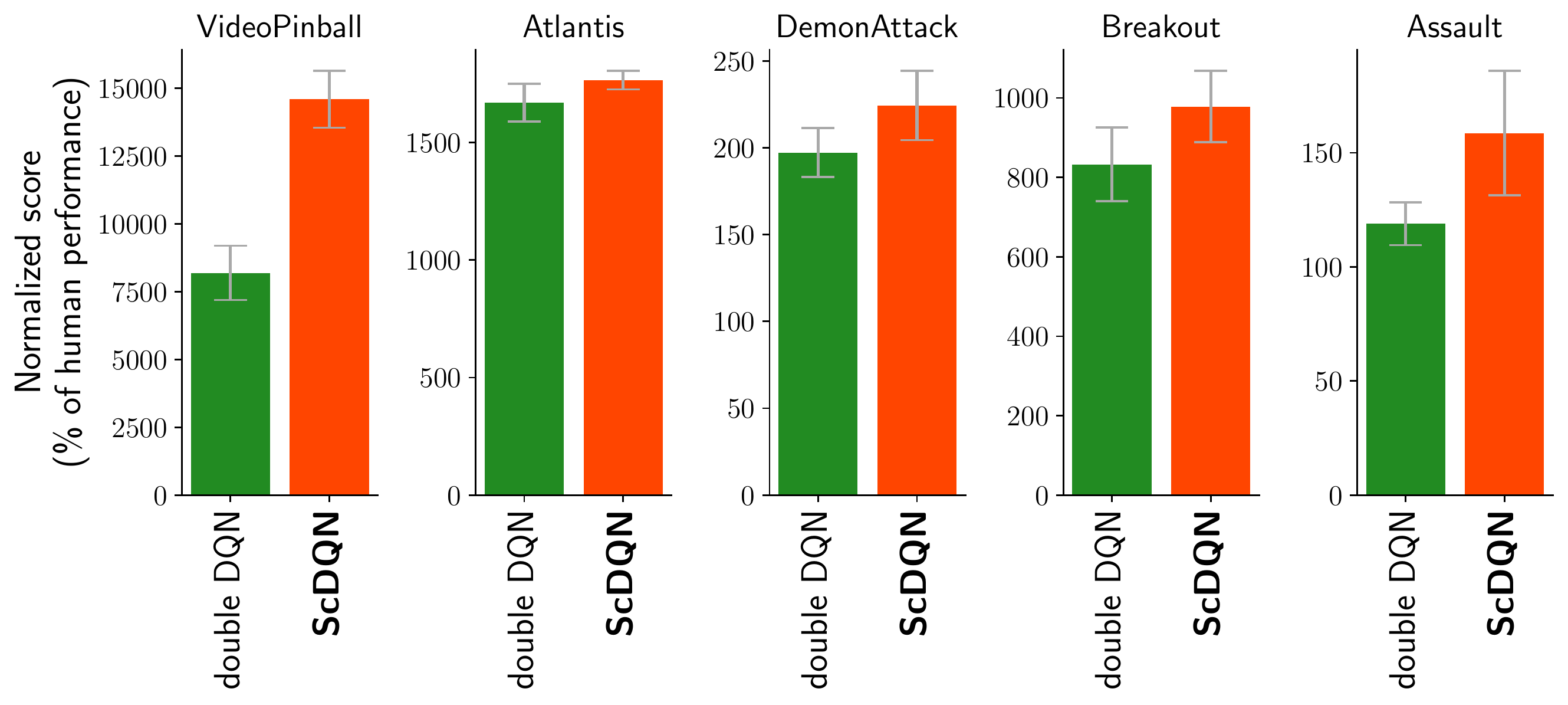}
    \caption{ScDQN rivals Double DQN on Atari 2600 games of DQN networks. Each bar indicates average score over 100 episodes and 6 random seeds (normalized to human performance \citep{Minh:15}), error bars are SEM over 6 random seeds. For ScDQN $\beta$ is set to $\beta$=3.0 for all games. Training is done for 20M steps  (see Appendix C), unless performance does not seem to stabilize (which was the case for \emph{VideoPinball} and \emph{Breakout}), in which case training is extended to 100M steps.}\label{fig:dqn_eval}
\end{figure*}

We study the Self-correcting Deep Q Network algorithm (ScDQN), i.e.\ the Neural Network version of Self-correcting Q-learning, in five representative tasks of the Atari 2600 domain: \emph{VideoPinball}, \emph{Atlantis}, \emph{DemonAttack}, \emph{Breakout} and \emph{Assault}. These games were chosen because they are the five Atari 2600 games for which Double DQN performes the best compared to human players \citep{HGS:16}.
We compare the performance of ScDQN against Double DQN \citep{HGS:16}.
We trained the same architecture presented in \cite{Minh:15} as implemented in Vel (0.4 candidate version, \cite{vel2018}).
Each experiment is run 6 times with different random seeds (as in e.g.\  \citep{HGS:16}), and we report average performance and variability across 6 independently trained networks.
In all experiments the network explores through $\epsilon$-greedy action selection.
The parameter $\epsilon$ starts off a $1.0$ and is linearly decreased to $0.1$ over 1M simulation steps, while $\beta$ is kept constant throughout.

We observed that our algorithm ScDQN has a faster and more stable convergence to a high reward solution than DQN and double DQN in the tested tasks, as we show in Fig.\ \ref{fig:videopinball} the representative example of the \emph{VideoPinball} task.
Interestingly, ScDQN also tends to display lower value estimates than Double DQN. We hypothesize that this might mainly be due to Double DQN being able to mitigate the underestimation problem of vanilla Double Q-learning thanks to the target network.
Fig.\ \ref{fig:dqn_eval} shows the final evaluation of DQN networks trained with Double DQN and ScDQN, and in all of the shown tasks ScDQN is at least as good as Double DQN.

%%%%%%%%%%%
\section{Conclusion and Discussion}
\label{sec:end}

We have presented a novel algorithm, \emph{Self-correcting Q-learning}, to solve the maximization bias of Q-learning. 
This method balances the the overestimation of Q-learning and the underestimation of Double Q-learning.
We demonstrated theoretical convergence guarantees for Self-correcting Q-learning, and showed that it can scales to large problems and continuous spaces just as Q-learning.

We studied and validated our method on several tasks, including a neural network implementation, ScDQN, which confirm that Self-correcting (deep) Q-learning reaches better performance than Double Q-learning, and converges faster than Q-learning when rewards variability is low.

One question left open is how to optimally set the new parameter $\beta$. Luckily, Self-correcting Q-learning does not seem sensitive to the choice of $\beta$.
In our experiments, $\beta=2-4$ is a good range.
Empirically, for larger state spaces and reward variability, larger $\beta$ tend to work better.
Future investigations on the effect of $\beta$ would still be welcome. 

Further interesting future research directions are:
(1) formal understanding of the advantage of Self-correcting Q-learning over Q-learning, as for instance in the cliff-walking task with fixed rewards.
(2) Besides Q-learning, the maximization bias exists in other reinforcement learning algorithms, such as the actor-critic algorithm.  \cite{FHM:18} applied the idea of Double Q-learning to the actor-critic algorithm.
Our self-correcting estimator could potentially be applied in a similar way.

\section*{Acknowledgements}
For this work, RZ was partially supported by the National Natural Science Foundation of China under grants 11871459 and 71532013. 
%Most of the work was done while the author worked at Columbia University.

%%%%%%%%%%%%%% REFERENCES   
\bibliography{refs}

%%%%%%%%%%%%%% APPENDIX
\appendix
\onecolumn

\section*{Appendix A}
\label{appendix}

%\subsection{Proof of Theorem \ref{them-SCQ}}

\begin{proof}[\textbf{Lemma \ref{lemma-compare}}]

Denoting $\tau=\frac{\beta-1}{\beta}$, from Eqn.\ (\ref{Qe}) it follows that 
\begin{equation}\label{Qe2}
    Q^{B_0}(a_i)=\tau~Q^{B_1}(a_i)+(1-\tau)Q^{B_{\tau}}(a_i).
\end{equation} 

For simplicity, we assume that $a_{\tau}^*=\arg\max_{a_i} Q^{B_\tau}(a_i)$ is unique (otherwise pick a maximum uniformly at random). 
From Eqn.~\eqref{Qe2},
we have that 
\begin{align}\label{expectation}
\text{E}[Q^{B_0}(a_{\tau}^*)]%&=\text{E}[\sum\limits_{i=1}^MQ^{B_1}(a_i)1\{a_i=a_{\tau}^*\}]\notag\\
&=\tau E[Q^{B_1}(a_{\tau}^*)]+(1-\tau)\text{E}[Q^{B_{\tau}}(a_{\tau}^*)]\notag\\
&=\tau E[Q^{B_1}(a_{\tau}^*)]+(1-\tau)\text{E}[\max_i Q^{B_{\tau}}(a_i)],%\notag\\
%&=\tau E[Q^{B}(a_D^*)]+(1-\tau)\text{E}[\max_i\{Q^{B(a_i)\}],
\end{align}
where the last step is from the definition of $a_{\tau}^*$. 
This implies that $\text{E}[Q^{B_1}(a_{\tau}^*)]$ takes value between  $E[Q^{B}(a_{\tau}^*)]$ and $\text{E}[\max_i Q^{B_{\tau}}(a_i)]$ since $0\leq \tau <1$.

Noticing that $a_{\tau}^*$ is chosen by maximizing $Q^{B_{\tau}}(a_i)$ and that $\{Q^{B_{\tau}}(a_i)\}_{i=1}^M$ is an unbiased set of estimators of $\{\text{E}[Q(a_i)]\}_{i=1}^M$, 
$\max_i Q^{B_{\tau}}(a_i)$ is the single estimator from the unbiased set $\{Q^{B_{\tau}}(a_i)\}_{i=1}^M$.
On the other hand, because $\{Q^{B_{\tau}}(a_i)\}_{i=1}^M$ is independent from $\{Q^{B_1}(a_i)\}_{i=1}^M$, then $Q^{B_1}(a_{\tau}^*)$ is a double estimator from two unbiased sets $\{Q^{B_{\tau}}(a_i)\}_{i=1}^M$ and $\{Q^{B_1}(a_i)\}_{i=1}^M$.

Based on the properties of the single estimator and the double estimator,  we have that 
\begin{align*}
&\text{E}[\max_i Q^{B_1}(a_{\tau})]\leq \max_i\text{E}[Q(a_i)] \text{, and }\notag\\
&\text{E}[\max_i Q^{B_{\tau}}(a_i)]\geq \max_i\text{E}[Q(a_i)].
\end{align*}
Combining these inequalities with Eq.~\eqref{expectation} for $\text{E}[Q^{B_0}(a_{\tau}^*)]$ we can thus conclude that there exists a $\tau$ (with a corresponding $\beta$) such that $\text{E}[Q^{B_0}(a_{\tau}^*)]=\max_i\text{E}[Q(a_i)]$.

\end{proof}

\begin{lemma}\label{lem-J94}
(Theorem 1 of \cite{JJS:94}) Consider a random iterative process $(\alpha_n, \Delta_n, F_n)$ such that, 
for $x\in X$, 
$$\Delta_{n+1}(x)=[1-\alpha_n(x)]\Delta_{n}(x)+\alpha_n(x)F_n(x).$$
%where $x\in X$.  
Let $P_n=\{X_0,\cdots,X_n;  \alpha_0, \cdots, \alpha_{n} ;F_0,\cdots,F_{n-1}\}$. 
Assume the following assumptions are satisfied: 
(a1) $X$ is a finite set; 
(a2) $\alpha_n(x)$ are assumed to be nonnegative and mutually independent given $P_n$ such that  $\sum_{n=1}^{\infty}\alpha_n(x)=\infty$ and $\sum_{n=1}^{\infty}\alpha_n(x)^2<\infty$; 
(a3) $\|\text{E}[F_n(x)|P_n]\|_{\text{max}}\leq \zeta\|\Delta_n\|_{\text{max}}$ where $\zeta\in(0,1)$, and $\|\cdot\|_{\text{max}}$ refers to a maximum norm; %$c_n$ converges to 0 with probability 1,  
(a4) $\text{Var}[F_n(x)|P_n]\leq c(1+\|\Delta_n\|_{\text{max}})^2$, where $c$ is some constant. Then 
$\Delta_{n}(x)$ converges to zero with probability one.
\end{lemma}

%%%%%%%%%%%%%%%%%%
%\begin{proof}[\textbf{Theorem \ref{them-SCQ}}]
\begin{proof}[\textbf{Sketch of the proof of Theorem \ref{them-SCQ}}]

We sketch how to apply the steps proving Lemma \ref{lem-J94} in \cite{JJS:94} to prove Theorem \ref{them-SCQ} without going into full
technical detail.
Denote $V^*(s)=\max_aQ^*(s,a)$, and $R(s,a)=\sum_{s^{\prime}}P_{ss^{\prime}}(a)R_{ss^{\prime}}(a)$, 
where $P_{ss^{\prime}}(a)$ is the probability of ending up in state $s^{\prime}$ after performing action $a$ in state $s$, 
and $R_{ss^{\prime}}(a)$ is the corresponding reward $r$. 
The \emph{Bellman equations} are rewritten as follows:
\begin{equation}\label{Bellman-v2}
Q^*(s,a)=R(s,a)+\gamma\sum\limits_{s^{\prime}}P_{ss^{\prime}}(a)V^*(s^{\prime}), \forall s,a.
\end{equation}
\iffalse %%%%
\begin{align*}
V_n(s)&=\max_aQ_n\left(s, a\right); \notag\\ 
V_n^{\beta}(s)&=Q_n\left(s^{\prime}, \hat{a}^{\beta}\right),
\text{ where }\hat{a}^{\beta}= \arg\max_a Q_n^{\beta}(s,a).
\end{align*}
Moreover, the iteration of Self-correcting Q-learning is rewritten as follows:
\begin{equation}\label{Bellman-n}
Q_{n+1}(s,a)=r+\gamma\sum\limits_{s^{\prime}}P_{ss^{\prime}}(a)V_n^{\beta}(s^{\prime}). 
\end{equation}
\fi %%%%

%We need to prove that $Q_{n+1}(s,a)$ in Eqn. (\ref{Bellman-n}) goes to $Q^*(s,a)$ in Eqn. (\ref{Bellman-v2}) with probability one. 
We need to prove that $Q_{n+1}(s,a)$ in Q-learning goes to $Q^*(s,a)$ in Eqn. (\ref{Bellman-v2}) with probability one. 
In order to prove the convergence, we apply Lemma \ref{lem-J94} above into $Q_n(s,a)$.  
%Theorem 3 of \cite{JJS:94} have verified those condition of Lemma \ref{lem-J94}. 
%Let $V_n^{\beta}(s)=Q_n\left(s^{\prime}, \hat{a}^{\beta}\right)$, where $\hat{a}^{\beta}= \arg\max_a Q_n^{\beta}(s,a)$.

Let $V_n(s^{\prime})=Q_n\left(s^{\prime}, \hat{a}^{\beta}\right)$, where 
%Let $\hat{a}^{\beta}= \arg\max_a Q_n^{\beta}(s^{\prime},a)$.
By identifying $X=S\times A$, $\Delta_n(s,a)=Q_{n}(s,a)-Q^*(s,a)$, $\alpha_n(x)=\alpha_n(s,a)$, 
$P_n=\{Q_0, s_0, a_0, \alpha_0, s_1, a_1, \alpha_1, \cdots, s_n, a_n, \alpha_n\}$
and $F_n(s,a)=r+\gamma Q_n\left(s^{\prime}, \hat{a}^{\beta}\right)-Q^*(s,a)$, 
the iteration of Self-correcting Q-learning is rewritten as follows:
\begin{equation}\label{ScQ-update}
\Delta_{n+1}(s,a)=(1-\alpha_n(s,a))\Delta_{n}(s,a)+\alpha_n(s,a)F_n(s,a). 
\end{equation}
We have that if the conditions (a1-a4) of Lemma \ref{lem-J94} are verified, then $\Delta_n(s,a)$ goes to zero with probability 1. 
We shall verify those conditions.  It is straightforward to show that the first two conditions (a1) and (a2) of Lemma \ref{lem-J94} hold from C1 and C2 in Theorem \ref{them-SCQ}.

Denoting $\Delta_n^{\beta}(s,a)=Q_{n}^{\beta}(s,a)-Q^*(s,a)$, we rewrite $\Delta_n(s,a)$ as 
\begin{equation}\label{DD}
\Delta_n(s,a)=\Delta_n^{\beta}(s,a)+\beta\Delta_n^{Q}(s,a),
\end{equation}
where $\Delta_n^{Q}(s,a)=Q_{n}(s,a)-Q_{n-1}(s,a)$. 
 
\iffalse %%%
From the definition of $\Delta_n^{\beta}$ in Eqn.~(\ref{DD}), Eqn.~(\ref{ScQ-update}) is rewritten as follows:
\begin{equation}\label{Bellman-n}
Q_{n+1}^{\beta}(s,a)=Q_n^{\beta}(s,a)+\alpha_n(s,a)(1-\beta)\left[r+\gamma Q_n\left(s^{\prime}, \hat{a}^{\beta}\right)-Q_n(s,a)\right]+\beta\left[Q_n(s,a)-Q_{n-1}(s,a)\right]. 
\end{equation}
\fi %%%
By defining $e_n(s,a)=F_n(s,a)-\text{E}[F_n^{\beta}(s,a)|P_n]$, 
we decompose the above iterative process of Self-correcting Q-learning into two parallel processes which are given by
\begin{align}%\label{equation-pair}
\Delta_{n+1}^{\beta}(s,a)&=(1-\alpha_n(s,a))\Delta_n^{\beta}(s,a)+\alpha_n(s,a)\text{E}[F_n^{\beta}(s,a)|P_n]; \label{pair1}\\
\Delta_{n+1}^{Q}(s,a)&=(1-\alpha_n(s,a))\Delta_n^{Q}(s,a)+\alpha_n(s,a)\beta^{-1}e_n(s,a).\label{pair2}
\end{align}

First we investigate Eqn.~(\ref{pair1}). 
Now we verify the conditions (a3) and (a4) on $F_n(s,a)$. 
\begin{align}\label{eqn:Fn}
F_n(s,a)=&r+\gamma Q_n\left(s^{\prime}, \hat{a}^{\beta}\right)-Q^*(s,a)\notag\\
=&r+\gamma \max_{a^{\prime}}Q_n^{\beta}(s^{\prime},a^{\prime})-Q^*(s,a)+\gamma\left[Q_n\left(s^{\prime}, \hat{a}^{\beta}\right)-Q_n^{\beta}\left(s^{\prime}, \hat{a}^{\beta}\right)\right]\notag\\
=&r+\gamma \max_{a^{\prime}}Q_n^{\beta}(s^{\prime},a^{\prime})-Q^*(s,a)+\gamma\beta\Delta_n^{Q}\left(s^{\prime}, \hat{a}^{\beta}\right).
%\notag\\
%=:&F_n^{\beta}(s,a)+\beta\gamma\left[Q_n\left(s^{\prime}, \hat{a}^{\beta}\right)-Q_{n-1}\left(s^{\prime}, \hat{a}^{\beta}\right)\right],
\end{align}
%where we define $F_n^{\beta}(s,a)=r+\gamma V_n^{\beta}(s^{\prime})-Q^*(s,a)$. 

%For the term $F_n(s,a)$ in Eqn.~(\ref{eqn:Fn}), %Eqns. (\ref{Bellman-v2}) and (\ref{Bellman-n}) follow that 
From Eqn.~(\ref{Bellman-v2}), Eqn.~(\ref{eqn:Fn}) follows that  
\begin{align}\label{Fmax}
\max_a\left|\text{E}[F_n^{\beta}(s,a)|P_n]\right|
=&\gamma \max_a\left|\sum_{s^{\prime}\in S}P_{ss^{\prime}}(a)(V_n(s^{\prime})-V^{*}(s^{\prime}))\right|\notag\\
\leq&\gamma \max_a\sum_{s^{\prime}\in S}P_{ss^{\prime}}(a)\left[|\max_{a^{\prime}}Q_n^{\beta}(s^{\prime},a^{\prime})-\max_{a^{\prime}}Q^{*}(s^{\prime},a^{\prime})|
+\beta|\Delta_n^{Q}\left(s^{\prime}, \hat{a}^{\beta}\right)|\right]\notag\\
%=& \gamma \max_a\left|\sum_{s^{\prime}\in S}P_{ss^{\prime}}(a)\left[Q_n(s^{\prime},\hat{a}^{\beta})-Q^{*}(s^{\prime},\hat{a}^{\beta})\right]\right|\notag\\
%\leq& \gamma \max_a\sum_{s^{\prime}\in S}P_{ss^{\prime}}(a)\left|\left[Q_n(s^{\prime},\hat{a}^{\beta})-Q^{*}(s^{\prime},\hat{a}^{\beta})\right]\right|\notag\\
\leq&\gamma \max_a\sum_{s^{\prime}\in S}P_{ss^{\prime}}(a)\left[\max_{a^{\prime}}|Q_n^{\beta}(s^{\prime},a^{\prime})-Q^{*}(s^{\prime},a^{\prime})|
+\beta\max_{a^{\prime}}|\Delta_n^{Q}\left(s^{\prime}, a^{\prime}\right)|\right]\notag\\
%\leq&\gamma \max_a\sum_{s^{\prime}\in S}P_{ss^{\prime}}(a)\max_{a^{\prime}}|Q_n(s^{\prime},a^{\prime})-Q^{*}(s^{\prime},a^{\prime})|\notag\\
%&+2\beta\gamma\max_a\sum_{s^{\prime}\in S}P_{ss^{\prime}}(a)\max_{a^{\prime}}|\Delta_n^{Q}\left(s^{\prime}, a^{\prime}\right)|\notag\\
%=:&\gamma \max_a\sum_{s^{\prime}\in S}P_{ss^{\prime}}(a)V^{\Delta_n}(s^{\prime})+\gamma\max_{a^{\prime}}|Q_n(s^{\prime},a^{\prime})-Q_{n-1}(s^{\prime},a^{\prime})
\leq&\gamma \max_{s^{\prime},a^{\prime}}|\Delta_n^{\beta}(s^{\prime},a^{\prime})|+\beta\gamma\max_{s^{\prime},a^{\prime}}|\Delta_n^{Q}(s^{\prime},a^{\prime})|,
%&\leq \gamma \max_a\sum_{s^{\prime}\in S}P_{ss^{\prime}}(a)\left[\max_{a^{\prime}}|Q_n(s^{\prime},a^{\prime})-Q^{*}(s^{\prime},a^{\prime})|+|Q_n(s^{\prime},a^{\prime}^{\tau})-\max_{a^{\prime}}Q_n(s^{\prime},a^{\prime})|\right]\notag\\
%&= \gamma \max_a\sum_{s^{\prime}\in S}P_{ss^{\prime}}(a)V^{\Delta_n}(s^{\prime})+\gamma \max_a\sum_{s^{\prime}\in S}P_{ss^{\prime}}(a)\kappa_n^{\tau},
\end{align}
where the first step is from Eqn.~(\ref{Bellman-v2}), the second step is direct, 
the third step is from the fact that for $v, w\in \mathbb{R}^d$, 
$|\max v-\max w|\leq |\max |v-w|$, and 
the last step is direct and from the fact $\sum_{s^{\prime}\in S}P_{ss^{\prime}}(a)=1$.  %as \cite{JJS:94} does,
From Eqn.~(\ref{Fmax}), Eqn.~(\ref{pair1}) follows that 
\begin{align}\label{pair1-v2}
|\Delta_{n+1}^{\beta}(s,a)|&\leq (1-\alpha_n(s,a))|\Delta_n^{\beta}(s,a)|+\gamma\alpha_n(s,a)\max_{s,a}\left[|\Delta_n^{\beta}(s,a)|
+\beta|\Delta_n^{Q}(s,a)|\right].
\end{align}
For the term $e_n(s,a)$,  following \cite{JJS:94}, the variance of $e_n(s,a)$ given $P_n$ depends on $Q_n(s,u)$ at mostly linearly and $\text{Var}(r)$ is bounded from the condition C3. 
Thus, $\text{Var}(e_n(s,a)|P_n)\leq c(1+\max_{s,a}|\Delta_n(s,a)|)^2$, where $c$ is some constant. 

%Note that Eqns.~(\ref{pair1-v2} \& (\ref{pair2}) 
Note that the conditions in Lemma \ref{lem-J94} are verified from the conditions in Theorem \ref{them-SCQ}. 
Because the pair of Eqns.~(\ref{pair1-v2}) \& (\ref{pair2}) has the same form as the equation pair in proving Lemma \ref{lem-J94} in \cite{JJS:94}, 
we can directly apply the steps of proving Lemma \ref{lem-J94} in \cite{JJS:94} into Eqns.~(\ref{pair1-v2}) \& (\ref{pair2}), 
implying that both $\Delta_{n+1}^{\beta}(s,a)$ and $\Delta_{n+1}^{Q}(s,a)$ converges to 0 with probability 1. 
Therefore, the theorem is proved.

\end{proof}

\section*{Appendix B}

\subsection{Robustness of Self-correcting Deep Q-learning to variations in \texorpdfstring{$\beta$}{beta} parameter}

We empirically verify the robustness of Self-correcting Deep Q-learning (ScDQN) to variations of the $beta$ parameter for the first few games in the Atari 2600 benchmark.
Each panel in Figure \ref{fig:beta_robustness} shows the performance of ScDQN on one game after training for 20M steps with $\beta$ varied to be $2.0$, $3.0$, $4.0$.
Analogously to Figure \ref{fig:dqn_eval}, each bar indicates the average score over 100 episodes and 6 random seeds at the end of training.
The figure demonstrates that the performance of the trained ScDQN is relatively stable to the value of $\beta$.

\begin{figure*}[htb]
\centering
    \includegraphics[width=0.8\linewidth]{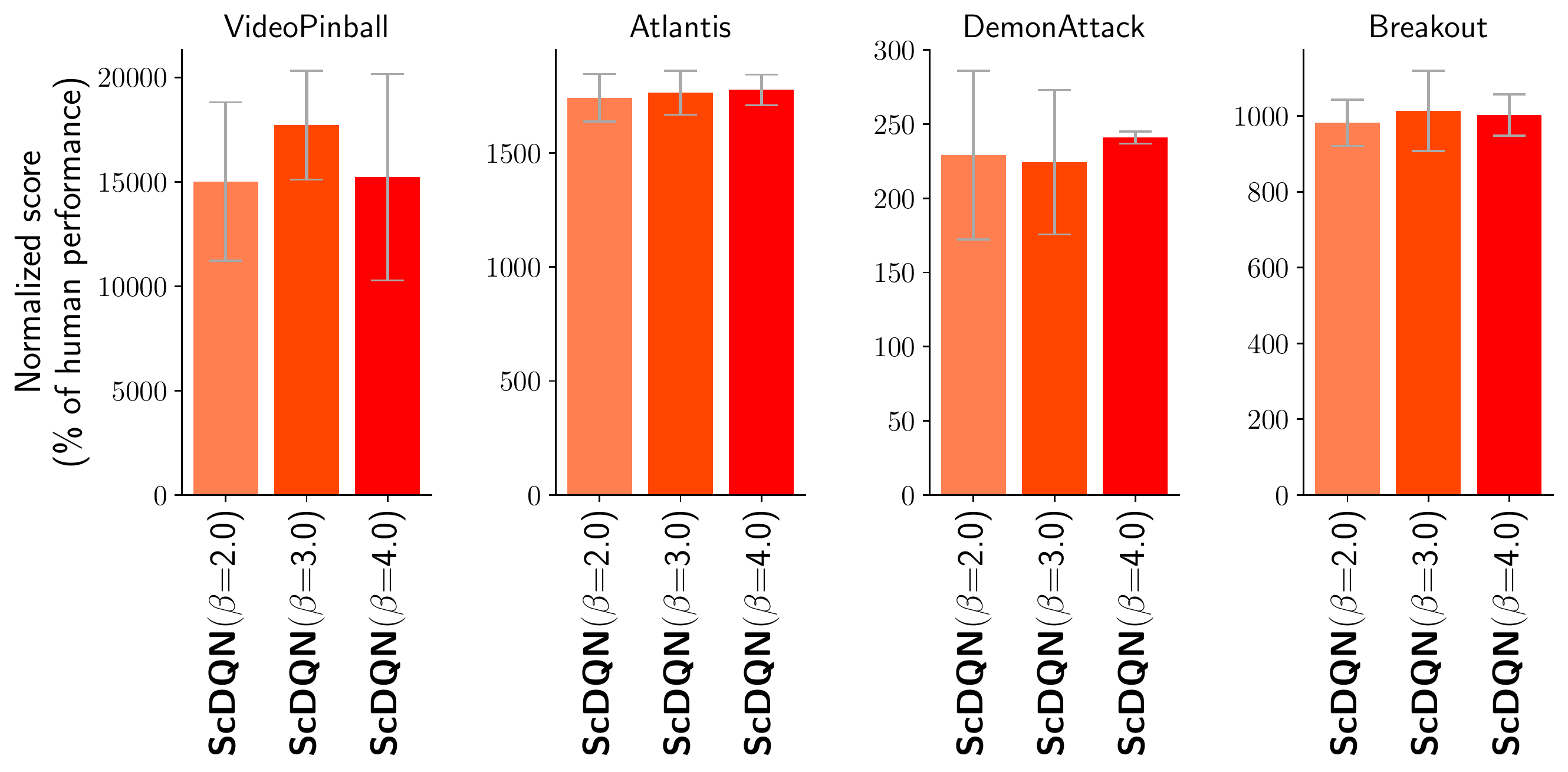}
    \caption{ScDQN is robust to variations of $\beta$ between $2.0$ and $4.0$ on Atari 2600 games. Each panel reports the performance of a DQN network trained for 20M steps with three values of $\beta$: 2.0, 3.0, 4.0. Each bar indicates average score over 100 episodes and 6 random seeds (normalized to human performance \citep{Minh:15}), error bars are SEM over 6 random seeds. The final performance of the trained DQN networks is relatively stable to variations of $\beta$ within the chosen range.}\label{fig:beta_robustness}
\end{figure*}

\clearpage

\section*{Appendix C}

\begin{figure*}[!h]
\centering
    \includegraphics[width=0.85\linewidth]{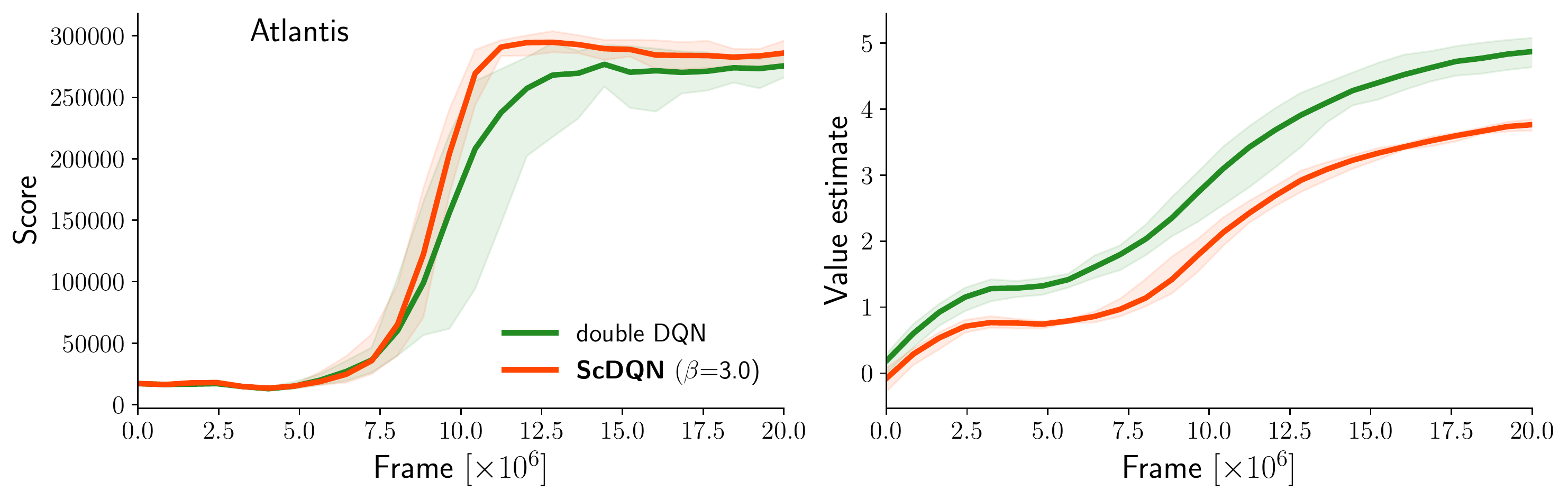}
    \includegraphics[width=0.85\linewidth]{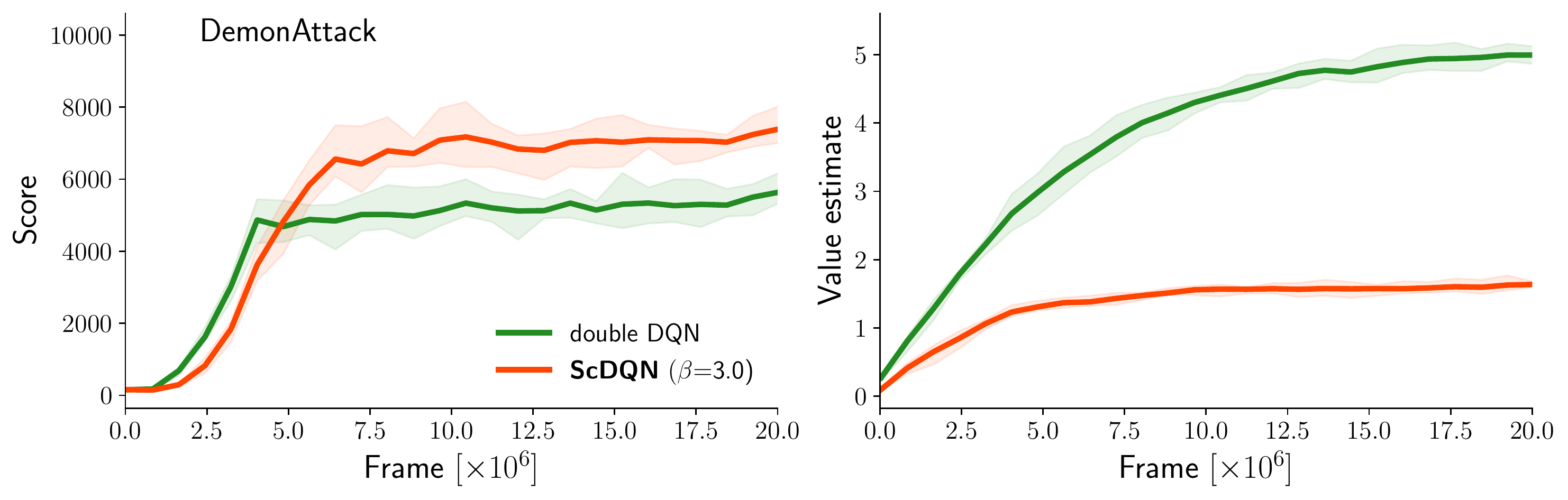}
    \includegraphics[width=0.85\linewidth]{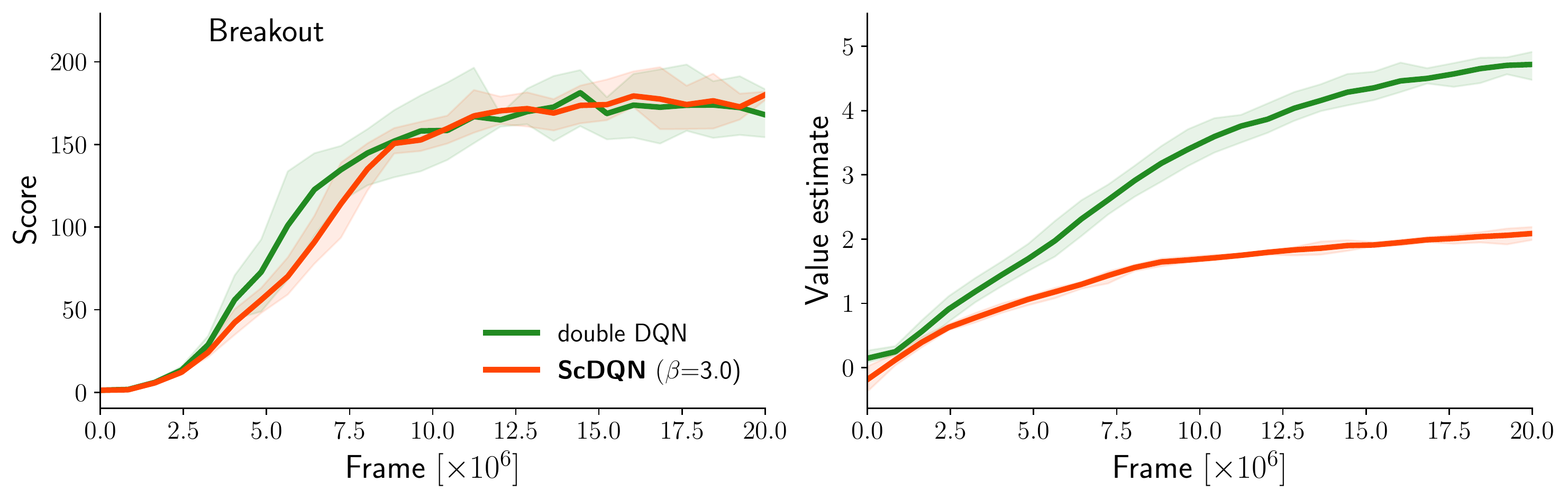}
    \includegraphics[width=0.85\linewidth]{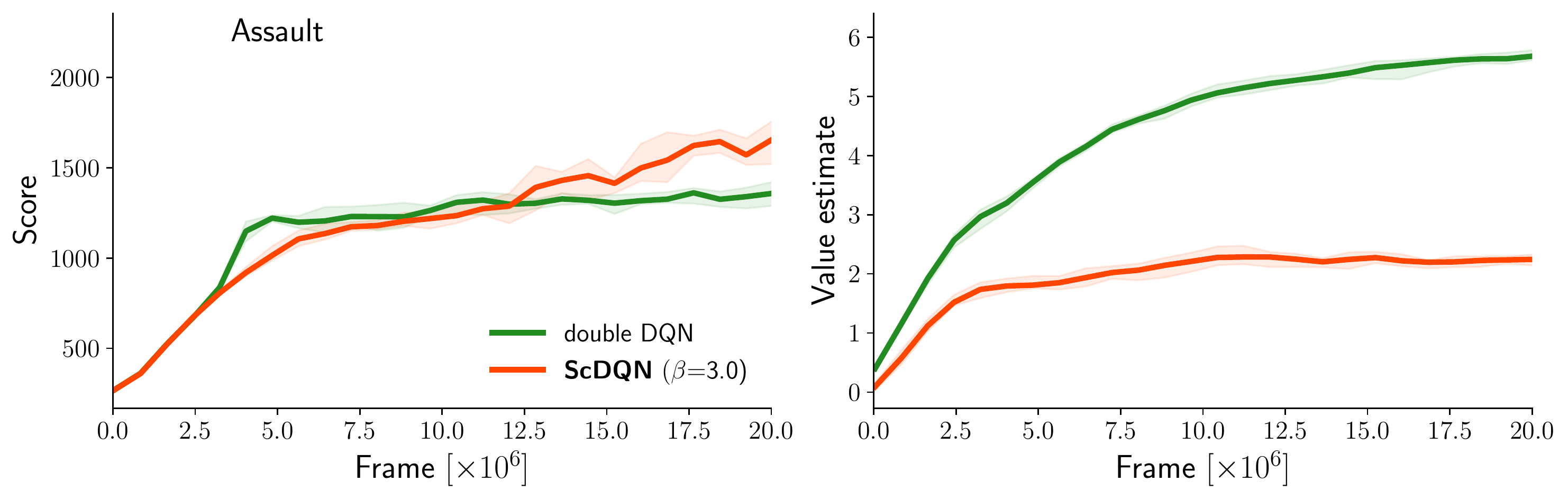}    
    \caption{Learning curves for Atari 2600 games evaluated in Fig.\ \ref{fig:dqn_eval} not shown in the main text of the paper. Results are obtained by running Double DQN, and ScDQN with 6 different random seeds and hyper-parameters from \cite{Minh:15}. Lines show average over runs and the shaded areas indicate the min and max values of average performance in intervals of 80000 frames over the 6 runs. \textbf{Left plots}: ScDQN consistently reaches scores that are comparable or higher than Double DQN. Double DQN considerably overestimates values measured as estimated value of the selected action compared to ScDQN (\textbf{right plots}).}\label{fig:add_learn}
\end{figure*}

\end{document}